\documentclass{article}

\usepackage[
  letterpaper,
  textwidth=5.5in,
  textheight=9in,
  centering
]{geometry}

\usepackage{times}
\usepackage[round,authoryear]{natbib}

\usepackage{amsmath,amssymb,amsthm,mathtools,bm}
\usepackage{graphicx}
\usepackage{booktabs,tabularx,multirow}
\usepackage{placeins}
\usepackage{enumitem}
\usepackage{microtype}
\usepackage{xcolor}
\usepackage{url}
\usepackage{float}
\usepackage[hidelinks]{hyperref}
\usepackage{fancyhdr}

\title{
  Activation-Flexible ANN-to-SNN Conversion\\
  with Finite-State Markov Neurons
}

\makeatletter
\renewcommand{\maketitle}{%
  \begin{center}

    {\LARGE\scshape \@title\par}

    \vspace{0.20in}

    {\large
      Ruiyu Jia$^{1,*}$
      \qquad
      Zhuo-Cheng Xiao$^{2,3,*}$
      \par
    }

    \vspace{0.10in}

    {\normalsize
      $^{1}$New York University Shanghai, Shanghai, China
      \par
    }

    \vspace{0.05in}

    {\normalsize
      $^{2}$NYU-ECNU Institute of Mathematical Sciences,\\
      New York University Shanghai, Shanghai 200124, China
      \par
    }

    \vspace{0.05in}

    {\normalsize
      $^{3}$NYU-ECNU Institute of Brain and Cognitive Science,\\
      New York University Shanghai, Shanghai 200124, China
      \par
    }

    \vspace{0.20in}

  \end{center}
}
\makeatother

\date{}

\newtheorem{theorem}{Theorem}
\newtheorem{corollary}{Corollary}
\newtheorem{remark}{Remark}

\newcommand{\ind}{\mathbf{1}}
\newcommand{\synops}{\operatorname{SynOps}}
\newcommand{\cliprelu}{\operatorname{ClipReLU}}
\newcommand{\E}{\mathbb{E}}
\newcommand{\Var}{\operatorname{Var}}
\newcommand{\gap}{\Delta_{\mathrm{acc}}}
\newcommand{\mf}{\mathrm{MF}}
\newcommand{\snn}{\mathrm{SNN}}
\newcommand{\ann}{\mathrm{ANN}}

\title{
  Activation-Flexible ANN-to-SNN Conversion\\
  with Finite-State Markov Neurons
}

\date{}

\begin{document}

\maketitle
\begingroup
\renewcommand{\thefootnote}{*}
\footnotetext{
  Corresponding authors:
  \texttt{rj2635@nyu.edu},
  \texttt{zx555@nyu.edu}.
}
\endgroup

\begin{abstract}
Most ANN-to-SNN conversion methods rely on a specific correspondence between the source activation and the spiking neuron dynamics. We propose a finite-state continuous-time Markov chain (CTMC) neuron framework whose stationary spike flux can approximate every continuous nonnegative monotone activation function on a compact interval. For a generalized CTMC family with affine input-dependent transitions, we prove uniform approximation to arbitrary accuracy over this function class and derive an explicit approximation error bound. In practice, two- and three-state CTMCs fit ReLU, sigmoid, softplus, and clipped ReLU on the evaluated input ranges, and we evaluate corresponding MLP conversions for each activation with layerwise rate scaling. Moderate clipping improves the conversion cost-accuracy tradeoff on the MNIST MLP and reduces SynOps by 27\% on VGG-11/MNIST at matched ANN-SNN accuracy gap criteria, whereas the trend reverses on VGG-11/CIFAR-10. Mean-field and layerwise diagnostics indicate that finite-window sampling and terminal-layer mismatch are the main residual errors. Overall, our results establish finite-state CTMC neurons as a theoretically grounded framework for activation-flexible ANN-to-SNN conversion beyond fixed activation-neuron correspondences.
\end{abstract}

\section{Introduction}
ANN-to-SNN conversion is now a standard route to deep spiking inference, but most current methods are restricted to ReLU-like activation functions and optimized for low-latency rate or time-based coding \citep{ding2021optimal,wang2022towards,stanojevic2023exact}. This motivates a broader question: can spiking neurons approximate arbitrary monotonic activation functions while retaining simple internal dynamics? We address this question with neurons that store only a discrete state and whose transition rates depend affinely on the input.

Inspired by earlier Markov neuron models and network reductions \citep{cai2021model,wu2023multiband,chang2025minimizing,wang2026finitestate}, we introduce finite-state continuous-time Markov chain (CTMC) neurons for ANN-to-SNN conversion. Their stationary spike flux approximates a target activation after layerwise rate scaling. The transition structure shapes the input--output curve, while refractory transitions control saturation. Our low-state neurons accurately fit several monotone activations on their operating ranges; a generalized CTMC family uniformly approximates any continuous, nonnegative monotone activation on a compact interval. We then examine how activation shape and spike budget affect conversion efficiency.

On an MNIST MLP, sigmoid, softplus, and ReLU convert within a 1\% accuracy gap, and moderate clipping reduces synaptic-event cost. On VGG-11/MNIST, ReLU6 improves conversion efficiency, with residual error mainly associated with finite-spike sampling. The CIFAR-10 results highlight the additional importance of layerwise distribution alignment. Together, these findings identify practical targets for improving conversion accuracy and efficiency.

We make four contributions:
    1. We present a practical ANN-to-SNN conversion framework based on finite-state Markov neurons, fitted stationary firing curves, and layerwise rate scaling.
    2. We show that boundedness can materially improve rate-coded conversion efficiency, most clearly through a clipped-ReLU ablation and a deep-MNIST VGG comparison.
    3. We use mean-field and layerwise diagnostics to separate fitting bias from finite-spike variance and to localize conversion mismatch in deep networks.
    4. In the appendix we prove that a generalized finite-state CTMC neuron with affine input-dependent transition rates can uniformly approximate any monotone nonnegative continuous activation on a compact interval.

\section{Related Work}
\paragraph{ANN-to-SNN conversion.}
Threshold balancing and rate normalization established deep conversion \citep{cao2015spiking,diehl2015fast,rueckauer2017conversion}. Later methods use residual potentials, quantized source activations, calibration, optimized activation-to-rate maps, or temporal codes to reduce error and latency \citep{han2020rmp,deng2021optimal,li2021free,ding2021optimal,bu2022optimal,wang2022towards,hao2023offset,jiang2023unified,stanojevic2023exact}; recent post-training work also emphasizes inference-scale conversion across broader vision tasks \citep{bu2025inference}. Recent work targets non-ReLU activations and few-spike approximations \citep{oh2024sign,jeong2024few}. We instead study activation flexibility under a compact finite-state stochastic representation.

\paragraph{Bounded activations and event-driven hardware.}
Explicit caps such as ReLU6 are used in efficient deep networks \citep{sandler2018mobilenetv2}. In conversion, however, a cap is useful only if its reduction in high-rate events outweighs any loss of source-model information. We test this tradeoff directly. Finite-state memory is compatible with event-driven implementation, but SynOps alone omits random-number generation, transition scheduling, memory access, and communication costs; therefore we do not claim measured hardware efficiency \citep{davies2018loihi}.

\section{Finite-State Markov Conversion}
\subsection{Rate Curves as Effective Activations}

For a rate-coded unit, the expected output rate acts as its effective activation. 
Ideal IF dynamics yield a positive linear current--rate map, whereas leak introduces 
a rheobase and curvature. More generally, the stationary spike-rate curve of a 
stochastic neuron determines which ANN nonlinearities it can represent. The following 
result shows that a generalized finite-state CTMC family is sufficiently expressive 
to approximate a broad class of activation functions to arbitrary accuracy.

\begin{theorem}[Uniform approximation]
\label{thm:universal}
Let $\phi:[I_{\min}, I_{\max}] \rightarrow [0,\infty)$ be continuous and nondecreasing. 
For every $\varepsilon > 0$, there exists a finite-state CTMC with affine 
input-dependent transition rates whose stationary spike flux $F$ satisfies
\[
\sup_{I \in [I_{\min}, I_{\max}]}
\left|F(I) - \phi(I)\right| < \varepsilon.
\]
\end{theorem}

The proof is given in Appendix~\ref{app:universal}. In practice, we use compact two- and three-state parameterizations for the activation functions considered here; we next describe the practical three-state construction.
\begin{figure}[!t]
  \centering
  \includegraphics[width=\linewidth]{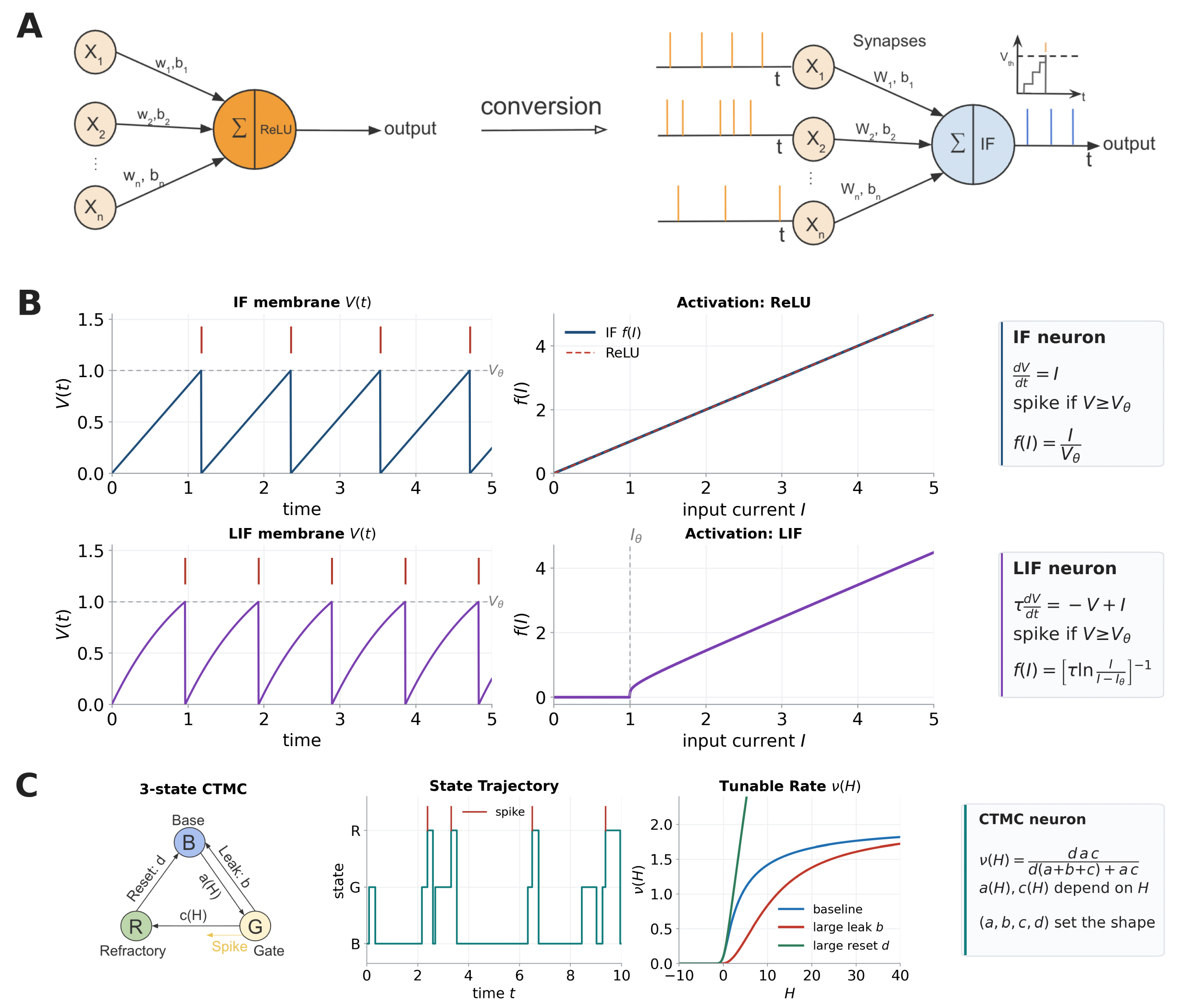}
  \caption{\textbf{From conventional neurons to a finite-state CTMC neuron.}
  \textbf{A:} Rate-coded ANN-to-SNN conversion at the level of one unit.
  \textbf{B:} IF and LIF membrane trajectories and their current-to-rate curves;
  IF aligns with the positive linear ReLU branch, whereas leak introduces a
  rheobase and curvature.
  \textbf{C:} The practical three-state neuron has base ($B$), gate ($G$), and
  refractory ($R$) states. Transitions $B\!\to\!G$, $G\!\to\!B$,
  $G\!\to\!R$ (spike), and $R\!\to\!B$ occur at rates $a(H)$, $b$, $c(H)$,
  and $d$, respectively. The displayed trajectories and parameter sweeps
  illustrate how leak and reset reshape the stationary firing curve $\nu(H)$.}
  \label{fig:overview}
\end{figure}

\subsection{Three-state neuron and stationary spike flux}
Let $X(t)\in\{B,G,R\}$ and write $a=a_\theta(H)\ge0$ and $c=c_\theta(H)\ge0$ for fitted input-dependent transition rates. The generator is
\begin{equation}
Q_\theta(H)=
\begin{bmatrix}
-a & a & 0\\
b & -(b+c) & c\\
d & 0 & -d
\end{bmatrix}, \qquad b,d>0.
\label{eq:generator}
\end{equation}
A spike is emitted on $G\to R$. Solving $\bm\pi Q=0$ and $\bm\pi\bm 1=1$ gives
\begin{equation}
\nu_\theta(H)=\pi_G(H)c_\theta(H)
=\frac{d\,a_\theta(H)c_\theta(H)}{d\,[a_\theta(H)+b+c_\theta(H)]+a_\theta(H)c_\theta(H)}.
\label{eq:rate}
\end{equation}
The practical model fits $\theta$ by least squares on a finite operating domain $\mathcal D$ sampled from ANN preactivations,
\begin{equation}
\min_{\theta}\ \frac{1}{M}\sum_{m=1}^M
\left[\nu_\theta(H_m)-\phi(H_m)\right]^2,
\qquad H_m\in\mathcal D,
\label{eq:fit}
\end{equation}
subject to nonnegative rates. All claims for ReLU concern compact empirical ranges; a finite-state stationary model is not claimed to match an unbounded function globally. Appendix~\ref{app:stationary} derives equation~\eqref{eq:rate}.

\begin{figure}[htbp]
  \centering
  \includegraphics[width=0.9\linewidth]{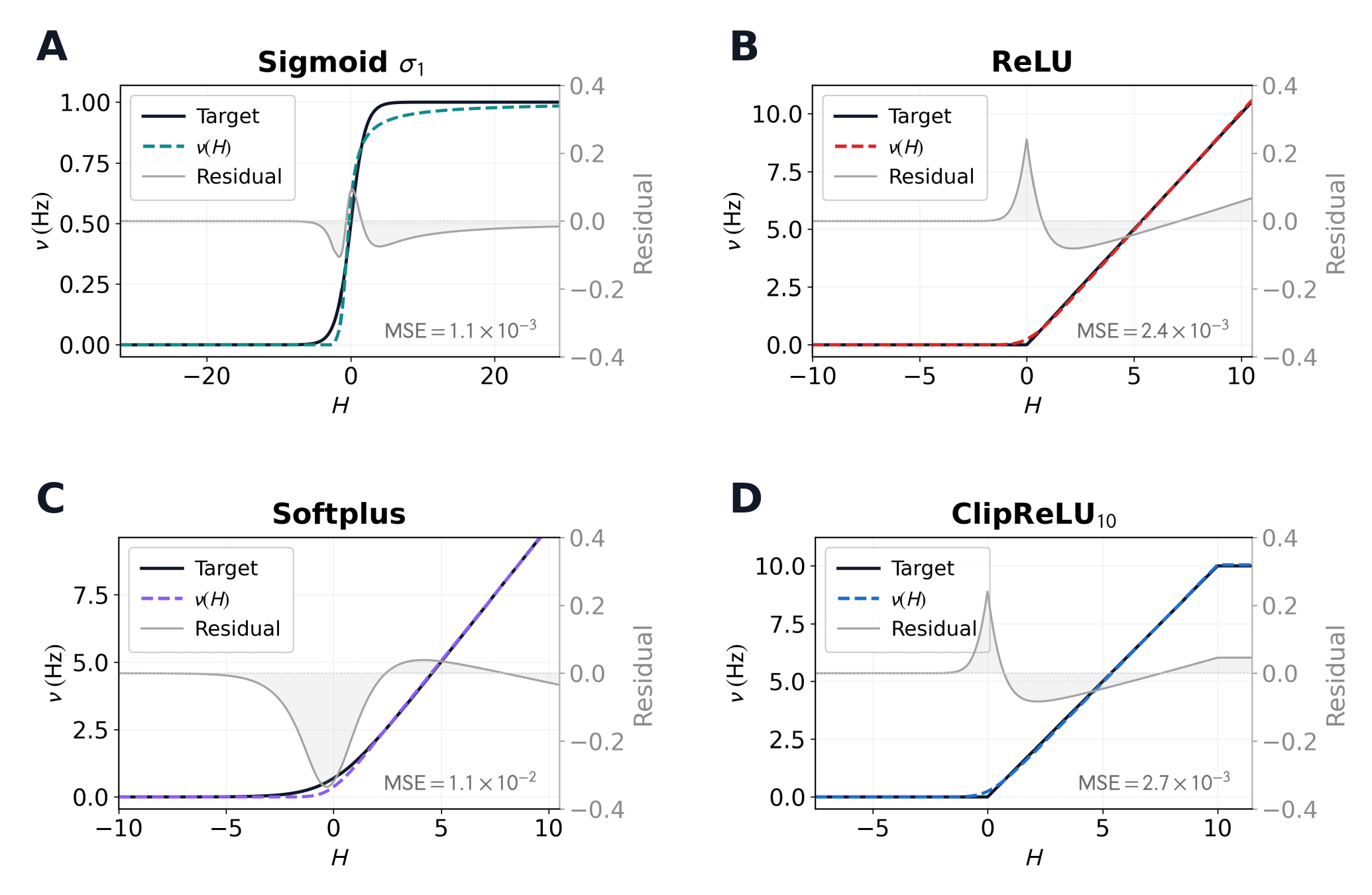} 
  \caption{\textbf{Stationary-rate fits on the displayed input domains.} Solid curves are target activations, dashed curves are fitted CTMC rates $\nu(H)$, and grey curves are residuals (right axis, shared range). The reported mean-squared errors are $1.1\times10^{-3}$ for sigmoid, $2.4\times10^{-3}$ for ReLU, $1.1\times10^{-2}$ for softplus, and $2.7\times10^{-3}$ for $\cliprelu_{10}$. The sigmoid uses the three-state family; the piecewise-linear/softplus examples use the low-state variant used in the implementation. These fits demonstrate finite-range flexibility, not exact global representation.}
  \label{fig:fits}
\end{figure}

\subsection{Layerwise scaling and finite-window decoding}
For ANN layer $\ell$,
\begin{equation}
 z^{\ell}=W^{\ell}a^{\ell-1}+b^{\ell},\qquad a^{\ell}=\phi_\ell(z^{\ell}),
\end{equation}
we select a positive scale $\alpha_\ell$ and target the rate $r^{\ell,*}=\alpha_\ell a^\ell$. To preserve the next affine map in expectation, incoming weights are rescaled as $\widetilde W^{\ell}=W^{\ell}/\alpha_{\ell-1}$. A generic filtered spike input is
\begin{equation}
H_j^\ell(t)=b_j^\ell+\sum_i\widetilde W_{ji}^\ell
\bigl(k_{\tau_H}*dN_i^{\ell-1}\bigr)(t),\quad
k_{\tau_H}(t)=\tau_H^{-1}e^{-t/\tau_H}\ind\{t\ge0\}.
\label{eq:trace}
\end{equation}
Because the kernel integrates to one, a stationary presynaptic rate $r^{\ell-1}$ yields the desired mean input. Over an inference window $T$, the decoded rate is $\widehat r_i^\ell=N_i^\ell(T)/T$ and $\widehat a_i^\ell=\widehat r_i^\ell/\alpha_\ell$. This construction separates the stationary mapping error from finite-$T$ sampling error.

\subsection{Evaluation quantities}
The accuracy gap is $\gap=A_\ann-A_\snn$ for the corresponding source ANN, with each accuracy expressed in percentage points. Thus $\gap$ is an absolute percentage-point difference, not a normalized relative error. For a sample $x$, the raw synaptic-event count is
\begin{equation}
\synops(x)=\sum_{\ell=1}^{L-1}\sum_i N_i^\ell(T;x)\,\operatorname{fanout}_i^{\ell+1}.
\label{eq:synops}
\end{equation}
We also report emitted spikes per neuron. For deep models, a deterministic mean-field network replaces sampled spike counts by $\nu_\theta(H)$. Its accuracy $A_\mf$ gives the exact algebraic decomposition
\begin{equation}
A_\ann-A_\snn=(A_\ann-A_\mf)+(A_\mf-A_\snn),
\label{eq:decomp}
\end{equation}
which we call \emph{mean-field mismatch} and \emph{finite-window sampling gap}, and is an operational accuracy decomposition.

\section{Experimental Design}
We evaluate MNIST \citep{lecun1998gradient} with an MLP of width $784$--$256$--$128$--$10$ and a strided-convolution VGG-11, and CIFAR-10 \citep{krizhevsky2009learning} with VGG-11 \citep{simonyan2015very}. The MLP experiments compare sigmoid, ReLU, softplus, a right-shifted sigmoid, and clipped ReLU. We define $\cliprelu_K(x)=\min\{\max(x,0),K\}$. In the clipping ablation, one trained ReLU MLP is reused and each forward pass is clamped to $[0,K]$; consequently each $K$ has its own ANN baseline. Deep experiments compare ReLU with $\cliprelu_6$ on MNIST and with a threshold derived from the empirical 95th percentile of the ReLU preactivation distribution on CIFAR-10.

Pareto frontiers vary the available simulation and rate parameters. MLP error bars are standard deviations over four seeds. VGG-11/MNIST reports means over four seeds and eight stochastic trials per seed; VGG-11/CIFAR-10 uses four seeds and four trials per seed. The plotted SEM across runs describes run-to-run variability, while trials sharing an ANN are nested within a training seed. Table~\ref{tab:headline} records the operating points highlighted by the supplied figures. Since each accuracy gap is referenced to its corresponding source ANN, clipped and unclipped rows compare conversion fidelity rather than absolute end-to-end accuracy.

\begin{table}[t]
\centering
\small
\caption{Highlighted operating points. SynOps are raw spike transmissions per sample. ``Criterion'' is the ANN--SNN accuracy gap in percentage points.}
\label{tab:headline}
\begin{tabular}{llccc}
\toprule
Setting & Activation & Criterion & SynOps/sample& Spikes/neuron\\
\midrule
MNIST MLP & sigmoid & $<1\%$ & $14.4{\pm}0.47\,$k & $0.90$\\
 MNIST    MLP& softplus&  $<1\%$ & $18.1{\pm}2.61\,$k&$1.02$\\
MNIST MLP & ReLU & $<1\%$ & $13.0{\pm}0.58\,$k & $0.87$\\
MNIST MLP & $\cliprelu_4$ & $<1\%$ & $9.1\,$k & $0.55$\\
VGG-11/MNIST & ReLU & $<0.5\%$ & $2.64\,$G & $12.97$\\
VGG-11/MNIST & $\cliprelu_6$ & $<0.5\%$ & $1.93\,$G & $9.51$\\
VGG-11/CIFAR-10 & ReLU & $<2\%$ & $2.06\,$G & $13.5$\\
VGG-11/CIFAR-10 & $\cliprelu_{q95}$ & $<2\%$ & $2.66\,$G & $15.7$\\
\bottomrule
\end{tabular}
\end{table}

\section{Results}
\subsection{Activation shape changes where the spike budget is spent}
Figure~\ref{fig:mlp}A shows that all sigmoid, softplus and ReLU reach a sub-1\% MLP gap, with ReLU doing so at the lowest SynOps. As for sigmoid and ReLU, the aggregate spike counts are similar (0.90 versus 0.87 spikes/neuron), but their distributions differ: sigmoid activity is denser and compressed, whereas ReLU is silent through much of the population and develops a heavy upper tail. A plausible source of the sigmoid cost is its nonzero output for negative preactivations. The shift-by-two control in Figure~\ref{fig:mlp}C is consistent with this interpretation but is budget dependent: the shift is worse at low budget and better only at sufficiently high budget, so it does not establish a universal benefit.

\begin{figure}[htbp]
  \centering
  \includegraphics[width=0.9\linewidth]{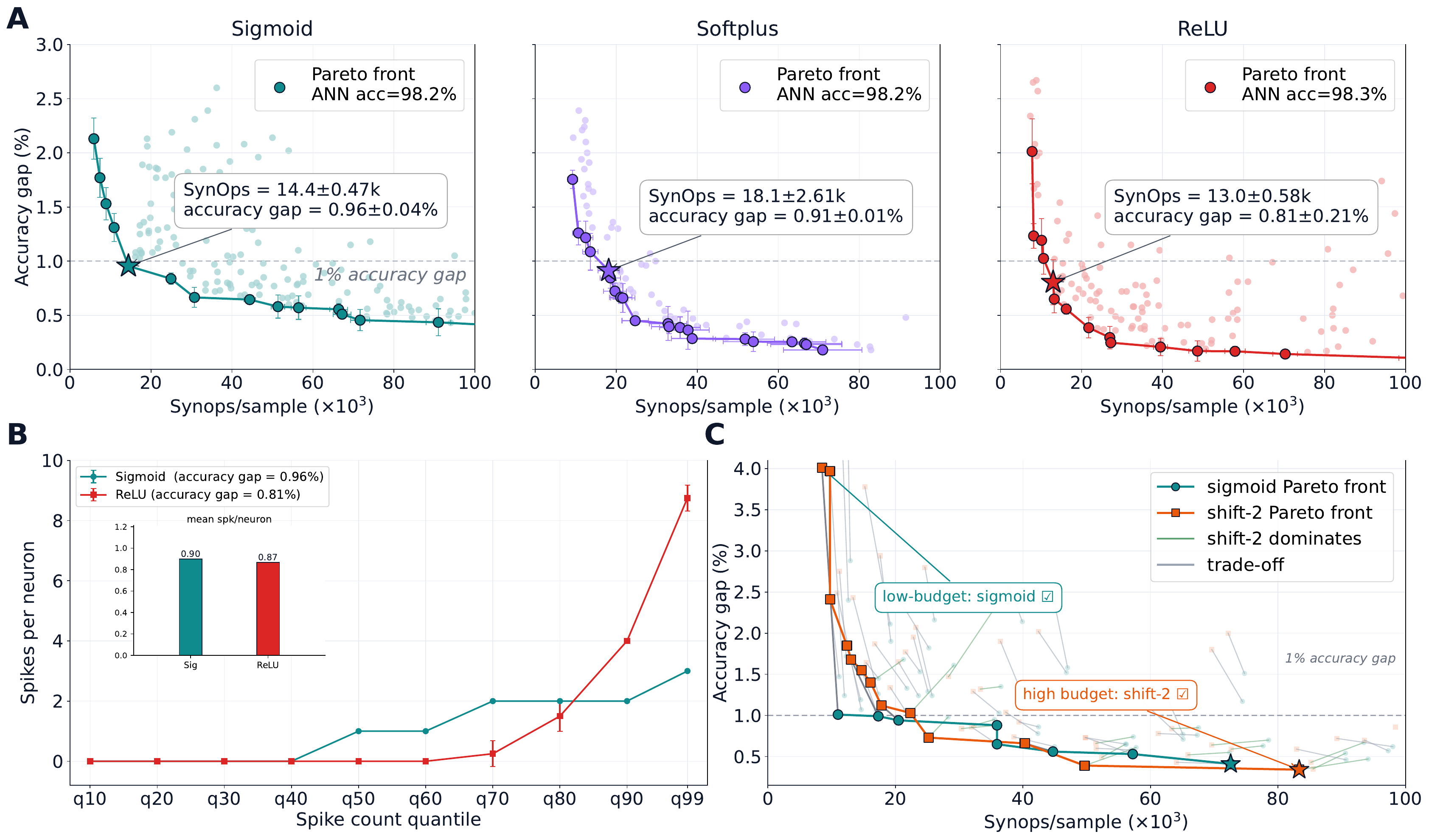}
  \caption{\textbf{MNIST MLP activation study} ($784$--$256$--$128$--$10$; error bars: four-seed standard deviations). \textbf{A:} Accuracy gap versus SynOps for sigmoid (source ANN 98.2\%), softplus (98.2\%) and ReLU (98.3\%); stars mark the minimum-cost configurations below a 1\% gap: $14.4\pm0.47$k SynOps at $0.96\pm0.04\%$ for sigmoid, $18.1\pm 2.61$k SynOps at $0.91\pm0.01\%$ for softplus and $13.0\pm0.58$k at $0.81\pm0.21\%$ for ReLU. \textbf{B:} Per-neuron spike-count quantiles show a compressed sigmoid distribution and a sparse, heavy-tailed ReLU distribution; mean counts are 0.90 and 0.87. \textbf{C:} Pareto frontiers for sigmoid and a two-unit right shift. Matched-configuration arrows show that suppressing the negative-input tail helps at high budget but can trade away accuracy at low budget.}
  \label{fig:mlp}
\end{figure}

\subsection{Moderate clipping improves the shallow frontier; aggressive clipping fails}
The clip-on-forward ablation in Figure~\ref{fig:clip} isolates a non-monotone tradeoff. The minimum SynOps needed for an accuracy gap below one percentage point are 16.0k, 9.1k, 9.8k, 11.1k, and 13.0k for $K=2,4,6,10,\infty$, respectively; $K=1$ never meets the criterion. Thus $K=4$ reduces the event count by 30\% relative to ReLU, while tighter clipping is harmful. Layerwise positive-preactivation quantiles and saturation fractions explain the transition: moderate caps primarily remove the high-rate tail, whereas small $K$ intersects the distribution body. At $K=1$, approximately 80\% of active layer-1 units and 39\% of active layer-2 units exceed the cap. This experiment supports a conditional mechanism---tail truncation can help---rather than a general advantage of bounded activations.

\begin{figure}[!t]
  \centering
  \includegraphics[
    width=0.95\linewidth
  ]{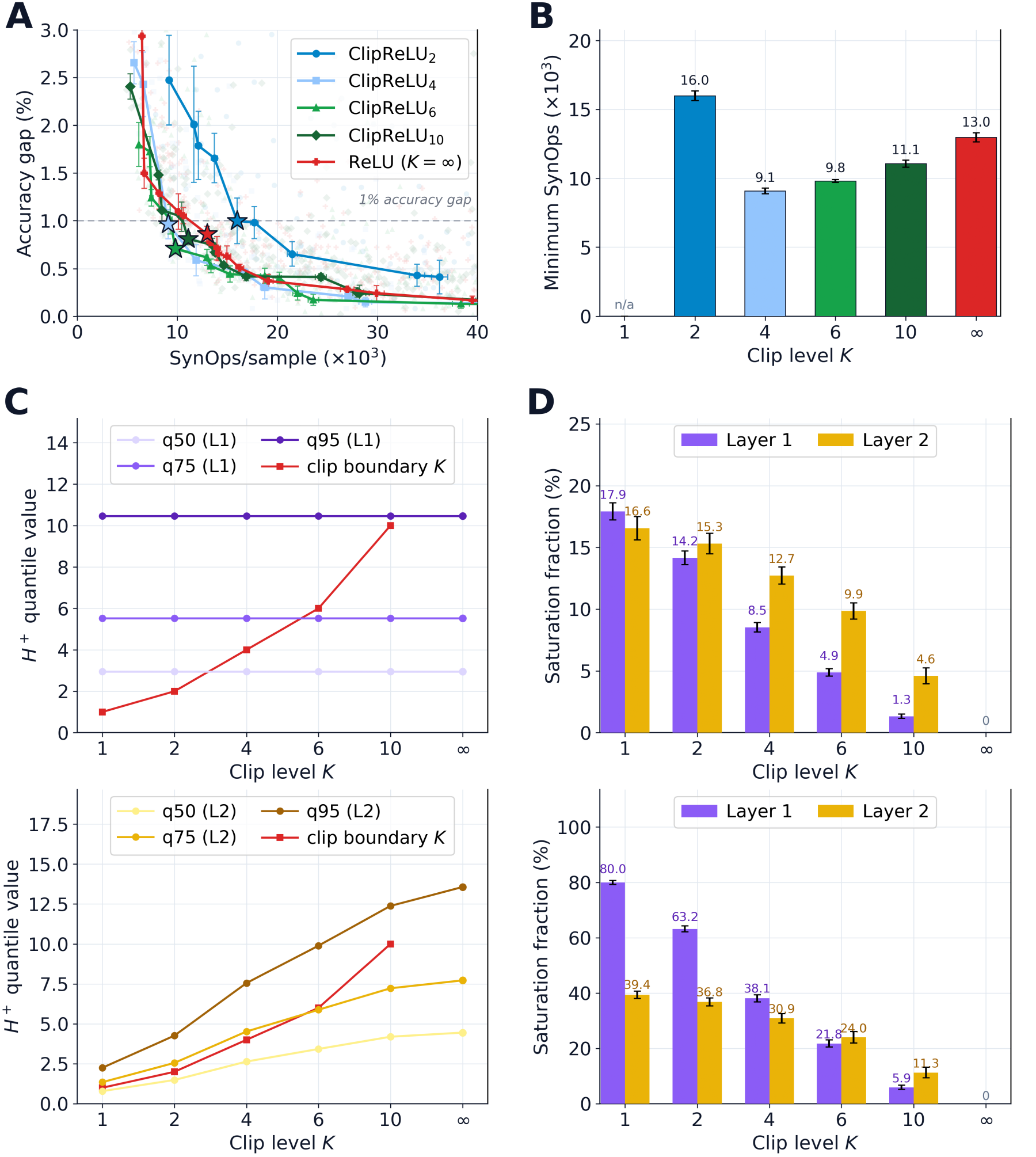}
 \caption{\textbf{$\operatorname{\mathbf{ClipReLU}}_{\mathbf K}$ ablation on the MNIST MLP, using shared ReLU weights and clip-on-forward evaluation.}
\textbf{A:} Accuracy-gap/SynOps frontiers.
\textbf{B:} Minimum SynOps for an accuracy gap below one percentage point; $K=4$ is best, and $K=1$ has no configuration satisfying the criterion.
\textbf{C:} Layerwise positive-preactivation quantiles with the clip boundary.
\textbf{D:} Fraction exceeding the cap, computed over all units (top) and active units (bottom).
Error bars denote standard deviations across four training seeds.
Because clipping changes the ANN forward map, each $K$ is evaluated against its own source accuracy.}
\label{fig:clip}
\end{figure}

\subsection{On VGG-11/MNIST, clipping lowers cost and sampling dominates the residual gap}
In Figure~\ref{fig:vggmnist}A, $\cliprelu_6$ reaches a sub-0.5\% gap with $1.93\times10^9$ SynOps and 9.51 spikes/neuron, compared with $2.64\times10^9$ and 12.97 for ReLU---approximately 27\% reductions in both quantities. Figure~\ref{fig:vggmnist}B shows that $A_\ann-A_\mf$ remains small across representative configurations, while $A_\mf-A_\snn$ is larger and varies with the finite simulation budget. Therefore the dominant residual in this experiment is finite-window stochastic sampling, not stationary-rate fit alone. Layerwise $q50/q90/q99$ traces in Figure~\ref{fig:vggmnist}C agree through early and middle layers and diverge mainly in the last two layers and upper quantiles; the bounded variant reduces this terminal mismatch.

\begin{figure}[!tbp]
  \centering
  \includegraphics[width=\linewidth,height=0.55\textheight,keepaspectratio]{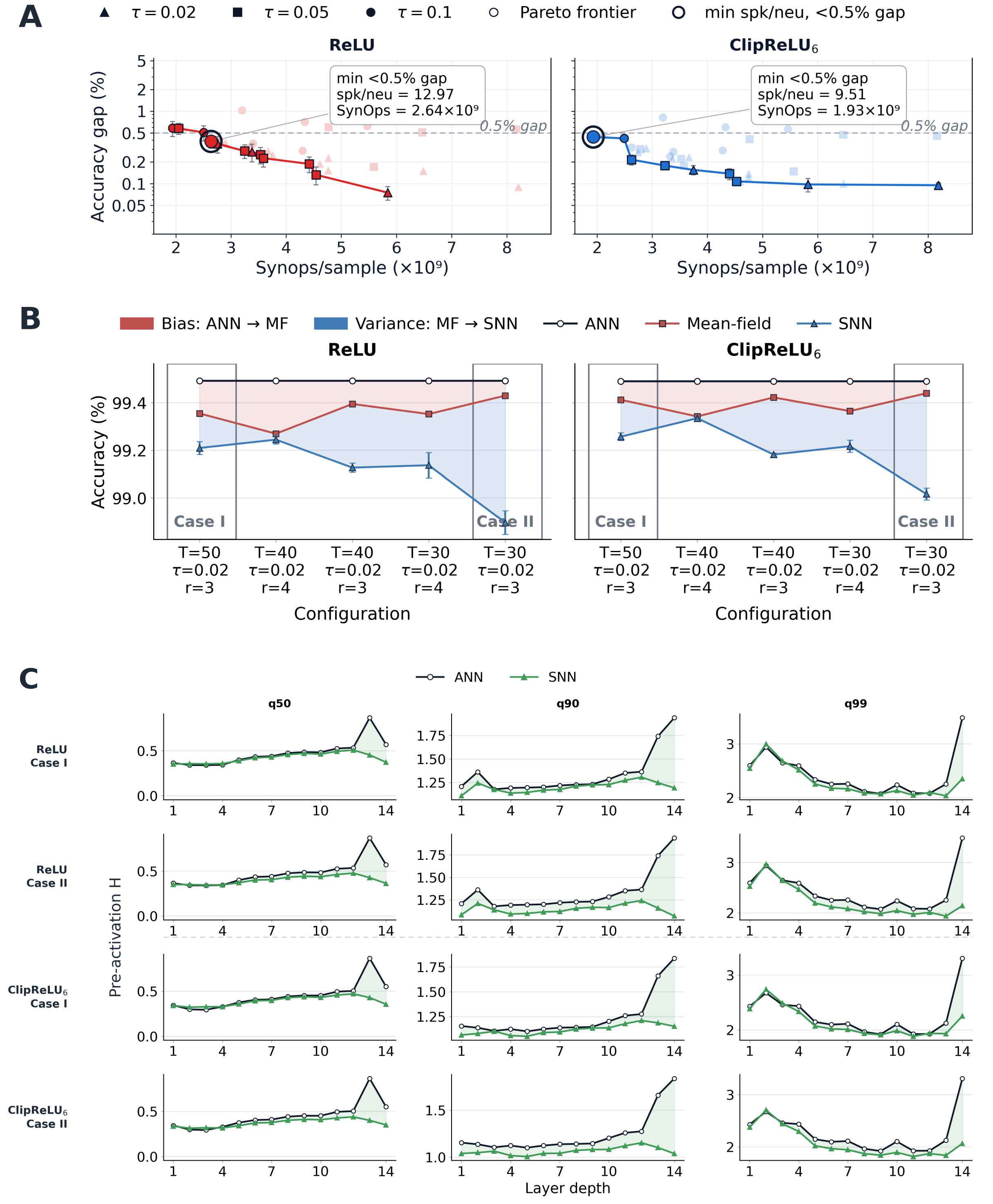}
  \caption{\textbf{VGG-11/MNIST: ReLU versus $\cliprelu_6$.} Means over four seeds and eight stochastic trials per seed; error bars are the supplied SEM across runs. \textbf{A:} Pareto frontiers; marker shape denotes $\tau$, and open circles mark the minimum-spike point below a 0.5\% gap. \textbf{B:} ANN, mean-field, and SNN accuracies; Cases I/II use $(T,\tau,r)=(50,0.02,3)$ and $(30,0.02,3)$. \textbf{C:} Preactivation quantiles across 14 layers.}
  \label{fig:vggmnist}
\end{figure}

\subsection{On CIFAR-10, the clipped variant is less efficient}
The VGG-11/CIFAR-10 result reverses the MNIST trend (Figure~\ref{fig:cifar}). ReLU reaches the sub-2\% band at $2.06\times10^9$ SynOps and 13.5 spikes/neuron, whereas $\cliprelu_{q95}$ requires $2.66\times10^9$ and 15.7. The clipped operating point therefore uses 29\% more SynOps. The same qualitative trend holds when cost is measured by CTMC state transitions on ReLU and $\cliprelu_{6}$ (Appendix~\ref{app:reporting}, Fig.~\ref{fig:cifar-transition-pareto}). For reference, conventional ReLU IF conversion shows a comparable accuracy--cost tradeoff on VGG-11/CIFAR-10; full results are provided in Appendix~\ref{app:baseline}.

The operational decomposition again assigns the larger component to finite-window sampling, especially at lower budget. Layerwise quantiles track well until the final layers, where upper-quantile underestimation is strongest for the clipped network. These observations are compatible with clipping removing useful upper-tail information, but they do not by themselves prove that causal explanation; a matched retraining and calibration study is needed.

\begin{figure}[!t]
  \centering
  \includegraphics[width=\linewidth,height=0.55\textheight,keepaspectratio]{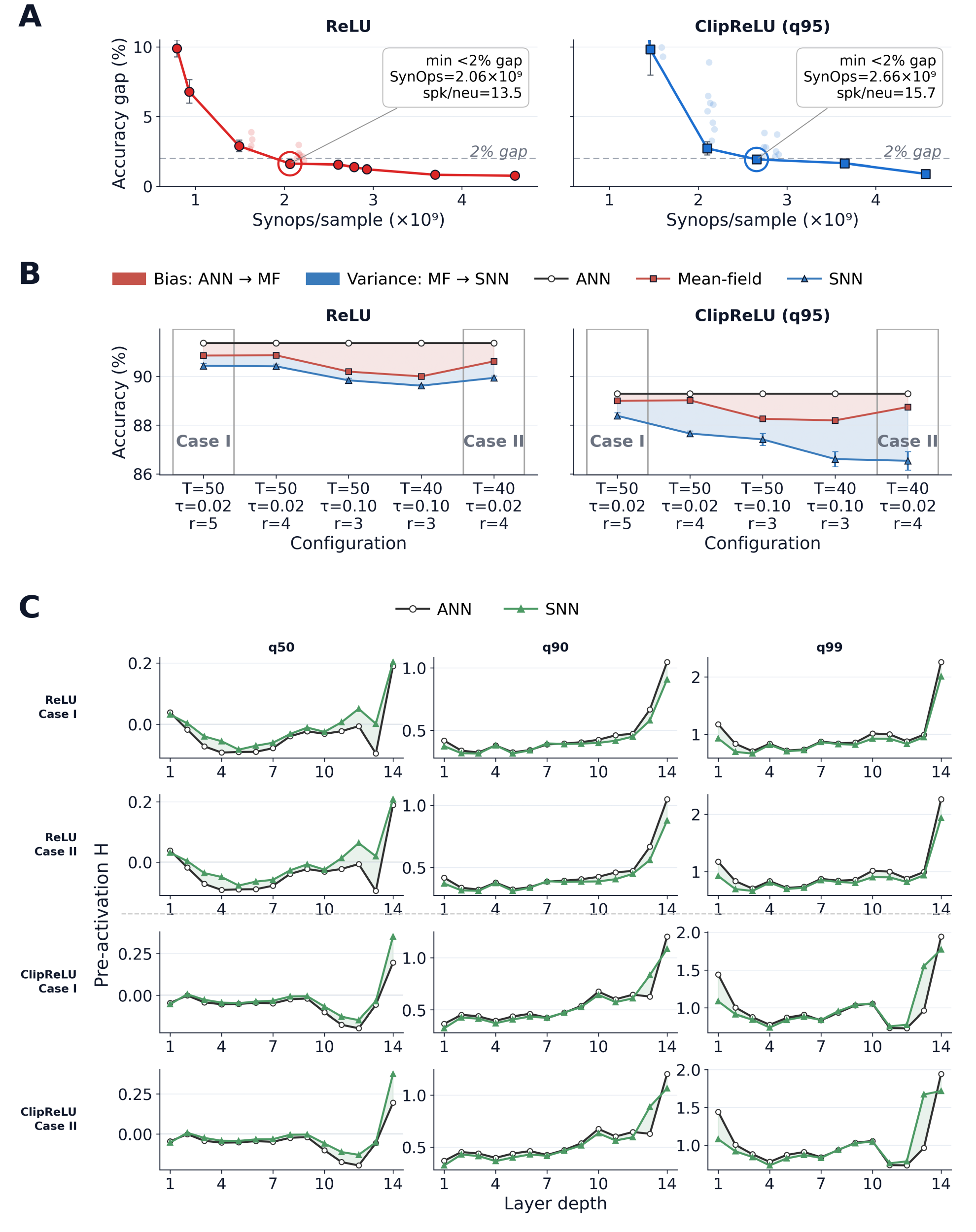}
  \caption{\textbf{VGG-11/CIFAR-10: ReLU versus a 95th-percentile clipped variant.} Means over four seeds and four stochastic trials per seed; error bars are the supplied SEM. \textbf{A:} Pareto frontiers and points below a 2\% gap. \textbf{B:} ANN, mean-field, and SNN accuracies. \textbf{C:} Preactivation quantiles across 14 layers for Cases I/II, $(T,\tau,r)=(50,0.02,5)$ and $(40,0.02,4)$; the largest discrepancy is terminal and high-quantile.}
  \label{fig:cifar}
\end{figure}

\section{Discussion}
Finite-state Markov neurons provide a common conversion framework for ANNs with different monotone activations. Combining stationary firing-rate approximation with layerwise scaling connects activation design to spike cost. The clipped-ReLU and VGG-11/MNIST experiments demonstrate that moderate clipping can reduce synaptic-event counts. The CIFAR-10 results complement this finding by identifying layerwise mismatch as a target for further improvement.

Mean-field and layerwise diagnostics offer a basis for jointly optimizing clipping thresholds, rate scaling, and spike allocation across layers. Future work can investigate how finite-window sampling errors propagate through deep networks and use this analysis to improve calibration. Evaluating event-driven hardware implementations would test whether the observed reductions in synaptic events also reduce inference latency and energy consumption.

\section{Conclusion}
Finite-state Markov neurons extend ANN-to-SNN conversion beyond ReLU-like activations through simple state dynamics and affine transition rates. The framework combines a general approximation guarantee with accurate low-state fits and demonstrated reductions in synaptic-event cost. These results provide a foundation for jointly designing activation functions, neuron dynamics, and layerwise calibration to support efficient spiking inference.

\FloatBarrier
\section*{Reproducibility Statement}
Equations~\eqref{eq:generator}--\eqref{eq:decomp} specify the model, scaling, decoding, and reported cost proxy; Appendices~\ref{app:stationary} and \ref{app:universal} give the analytical derivations, including the stationary firing rate and the universal approximation result. Appendices \ref{app:baseline} and \ref{app:reporting} provide additional baseline and diagnostic results. The supplementary material includes an anonymous code archive containing selected implementation code supporting the experiments and analyses reported in the paper.
\section*{Ethics Statement}
The reported experiments use standard image-classification benchmarks and do not involve human participants or sensitive personal data. The principal risk is over-interpreting SynOps as physical energy; throughout the paper it is described only as a raw event-transmission proxy.

\IfFileExists{iclr2027_conference.bst}{%
  \bibliographystyle{iclr2027_conference}%
}{%
  \IfFileExists{iclr2026_conference.bst}{%
    \bibliographystyle{iclr2026_conference}%
  }{%
    \bibliographystyle{plainnat}%
  }%
}
\bibliography{refs_iclr}

\appendix

\section{Stationary Rate of the Three-State CTMC}
\label{app:stationary}
Let $\bm\pi=(\pi_B,\pi_G,\pi_R)$ denote the stationary distribution of equation~\eqref{eq:generator}. The gate and refractory balance equations are
\begin{equation}
\pi_B a=\pi_G(b+c),\qquad \pi_Gc=\pi_Rd.
\end{equation}
Hence $\pi_B=\pi_G(b+c)/a$ and $\pi_R=\pi_Gc/d$. Normalization yields
\begin{equation}
\pi_G=\left(\frac{b+c}{a}+1+\frac{c}{d}\right)^{-1}.
\end{equation}
Multiplication by the spike-transition rate $c$ gives
\begin{equation}
\nu(H)=\pi_Gc=\frac{acd}{d(a+b+c)+ac},
\end{equation}
which is equation~\eqref{eq:rate}. This calculation assumes an irreducible chain at the evaluated input; boundary inputs with zero rates are obtained by continuity or by restricting the fitting domain to positive-rate parameters \citep{norris1998markov}.

\section{Universal Approximation by a General Finite-State CTMC}
\label{app:universal}

The low-state neurons used in our experiments are deliberately compact. 
To prove Theorem~\ref{thm:universal}, we consider a more general finite-state 
CTMC construction that addresses expressivity in principle while keeping every 
input-dependent transition rate affine in the input.

Let $\phi:[I_{\min},I_{\max}]\to[0,\infty)$ be continuous and nondecreasing. 
Normalize the input as
\begin{equation}
u=\frac{I-I_{\min}}{I_{\max}-I_{\min}}\in[0,1],
\qquad
\psi(u)=\phi\bigl(I_{\min}+(I_{\max}-I_{\min})u\bigr).
\end{equation}

Consider active states $A_0,\ldots,A_N$ with birth--death transitions
\begin{align}
A_i&\to A_{i+1} &&\text{at rate }\gamma(N-i)u,\\
A_i&\to A_{i-1} &&\text{at rate }\gamma i(1-u),
\end{align}
with absent boundary transitions omitted. Attach a refractory state $R_i$ to 
each active state:
\begin{equation}
A_i\xrightarrow{\kappa_i,\,\mathrm{spike}}R_i,
\qquad
R_i\xrightarrow{\rho}A_i,
\qquad
\kappa_i\ge0,\ \rho>0.
\end{equation}
Since $u$ is affine in $I$, all input-dependent transition rates are also affine 
functions of $I$.

The active birth--death chain has binomial stationary weights
\begin{equation}
b_i(u)=\binom{N}{i}u^i(1-u)^{N-i}.
\end{equation}
Writing
\begin{equation}
M_N(u)=\sum_{i=0}^{N}\kappa_i b_i(u),
\end{equation}
direct balance and normalization give the stationary spike flux
\begin{equation}
F_N(u)=\frac{\rho M_N(u)}{\rho+M_N(u)}.
\label{eq:bernrate}
\end{equation}

\begin{proof}[Proof of Theorem~\ref{thm:universal}]
Choose $\rho>\max_{u\in[0,1]}\psi(u)$ and define
\begin{equation}
g(z)=\frac{\rho z}{\rho+z}.
\end{equation}
On the target range, $g$ is invertible, with
\begin{equation}
h(u)=g^{-1}(\psi(u))
=\frac{\rho\psi(u)}{\rho-\psi(u)}.
\end{equation}
Because $\psi$ is continuous, nonnegative, and nondecreasing, so is $h$. 
Set $\kappa_i=h(i/N)$. Then $M_N$ in equation~\eqref{eq:bernrate} is the 
Bernstein polynomial of $h$:
\begin{equation}
M_N(u)
=B_Nh(u)
=\sum_{i=0}^{N}
h(i/N)\binom{N}{i}u^i(1-u)^{N-i}.
\end{equation}
By the Bernstein approximation theorem,
\begin{equation}
\|B_Nh-h\|_\infty\to0
\end{equation}
as $N\to\infty$ \citep{lorentz1986bernstein}. Moreover,
\begin{equation}
g'(z)=\frac{\rho^2}{(\rho+z)^2}\le1,
\end{equation}
so $g$ is $1$-Lipschitz on $[0,\infty)$. Therefore,
\begin{align}
|F_N(u)-\psi(u)|
&=|g(B_Nh(u))-g(h(u))|\\
&\le |B_Nh(u)-h(u)|.
\end{align}
The right-hand side converges uniformly to zero. Hence, for every 
$\varepsilon>0$, there exists $N$ such that
\begin{equation}
\sup_{u\in[0,1]}|F_N(u)-\psi(u)|<\varepsilon.
\end{equation}
Applying the affine change of variables between $u$ and $I$ gives
\begin{equation}
\sup_{I\in[I_{\min},I_{\max}]}
\left|
F_N\!\left(
\frac{I-I_{\min}}{I_{\max}-I_{\min}}
\right)-\phi(I)
\right|
<\varepsilon,
\end{equation}
which proves the result.
\end{proof}

\begin{corollary}[One quantitative bound]
If $h$ is $L_h$-Lipschitz, then
\begin{equation}
\|F_N-\psi\|_\infty\le\frac{L_h}{2\sqrt N}.
\end{equation}
Consequently, $N\ge(L_h/(2\varepsilon))^2$ is sufficient for error at most $\varepsilon$.
\end{corollary}

\begin{proof}
For $K\sim\operatorname{Binomial}(N,u)$, $B_Nh(u)=\E[h(K/N)]$. Hence
\begin{align}
|B_Nh(u)-h(u)|
&\le L_h\E|K/N-u|\\
&\le L_h\sqrt{\Var(K/N)}
=L_h\sqrt{u(1-u)/N}
\le\frac{L_h}{2\sqrt N}.
\end{align}
Apply the 1-Lipschitz property of $g$ used above.
\end{proof}

\begin{remark}
The theorem is an expressivity statement, not a practical complexity guarantee. The main experiments use two- or three-state parameterizations, whereas the construction uses $2(N+1)$ active-plus-refractory states. Establishing tight state, spike, and finite-time complexity for useful activation classes remains open.
\end{remark}
\paragraph{Fully connected three-state example.}
The universal approximation result concerns a generalized finite-state CTMC family,
whereas the main experiments use deliberately restricted low-state parameterizations.
As an intermediate example, we also fit a fully connected three-state CTMC in which
transitions are allowed between all pairs of states. Figure~\ref{fig:softplus-fullstate-n3}
shows that this additional transition flexibility substantially improves the softplus fit.

\begin{figure}[H]
  \centering
  \includegraphics[width=0.72\linewidth]{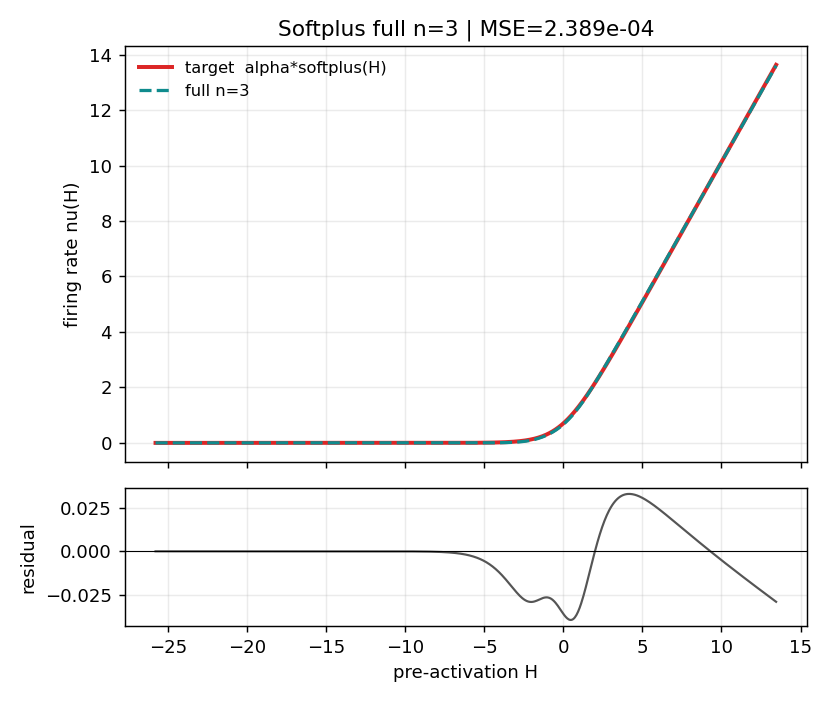}
  \caption{%
    Fully connected three-state CTMC fit to the softplus transfer curve.
    Unlike the restricted low-state CTMC used in the main experiments, this
    model permits transitions between all pairs of states, with a spike counted
    on transitions into the designated "base" state.
    \emph{Top:} target rate $\alpha\,\mathrm{softplus}(H)$ and the stationary
    spike-rate prediction over $H\in[-26,14]$.
    \emph{Bottom:} residual between the fitted and target rates.
    The residual remains within $\pm0.03$\,Hz over the evaluated range, with
    mean-squared error $2.39\times10^{-4}$.
  }
  \label{fig:softplus-fullstate-n3}
\end{figure}

\section{Baseline}
\label{app:baseline}

\begin{figure}[H]
  \centering
  \includegraphics[width=\linewidth]{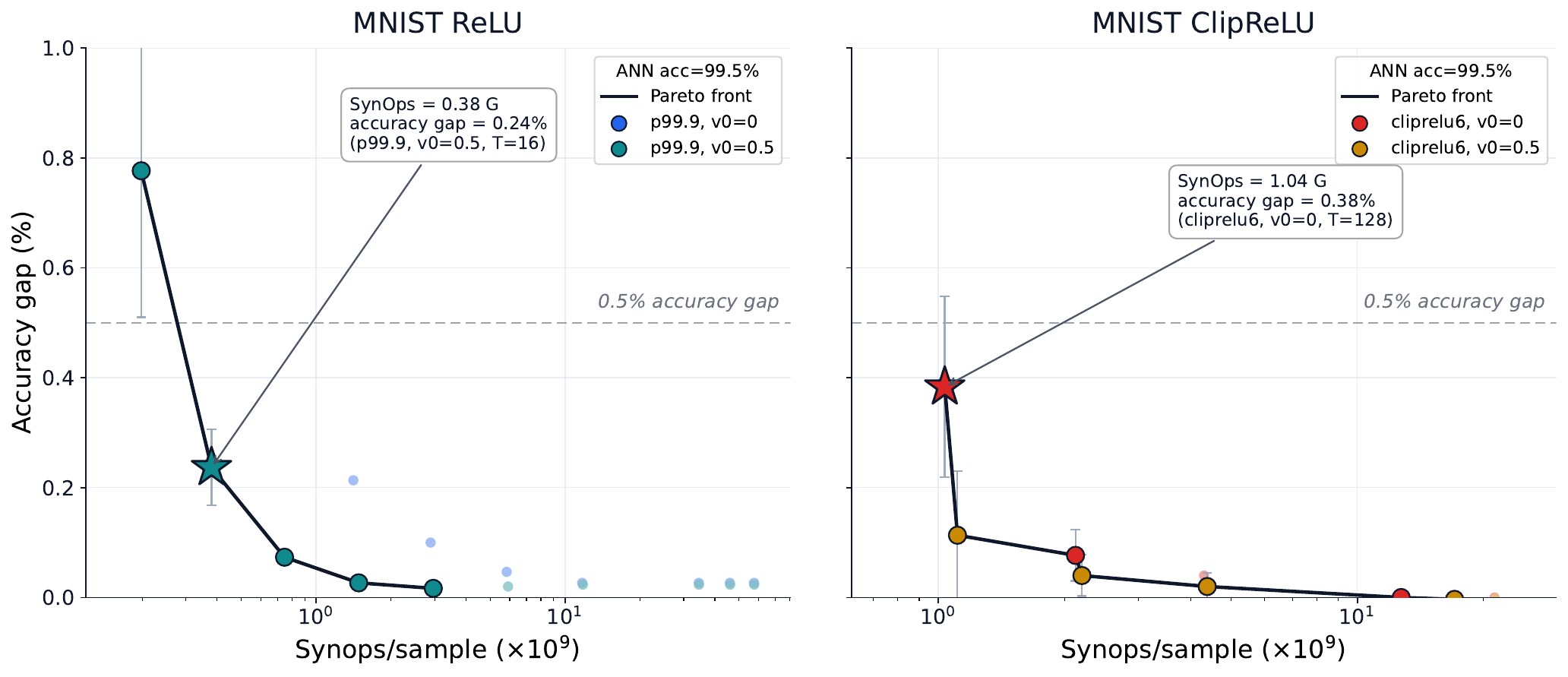}
  \caption{%
    Accuracy--cost trade-off for conventional integrate-and-fire (IF)
    conversion on VGG-11/MNIST, comparing the ReLU (left) and ClipReLU6
    (right) source networks. Each point represents one conversion
    configuration averaged over three seeds. Curves sweep the simulation length $T$ under the
    indicated normalization and initial-membrane settings
    $v_0 \in \{0,0.5\}$. The Pareto front shows the lowest attainable gap at
    each cost, with error bars denoting across-seed standard deviation. The
    dashed line marks the $0.5\%$ gap criterion, and the star marks the
    lowest-SynOps configuration satisfying it. The selected ReLU configuration
    reaches a $0.24\%$ gap at $0.38\,$G SynOps
    ($p99.9$, $v_0{=}0.5$, $T{=}16$), while ClipReLU6 reaches a $0.38\%$ gap
    at $1.04\,$G SynOps ($v_0{=}0$, $T{=}128$).%
  }
  \label{fig:if-baseline-mnist}
\end{figure}

\begin{figure}[H]
  \centering
  \includegraphics[width=0.62\linewidth]{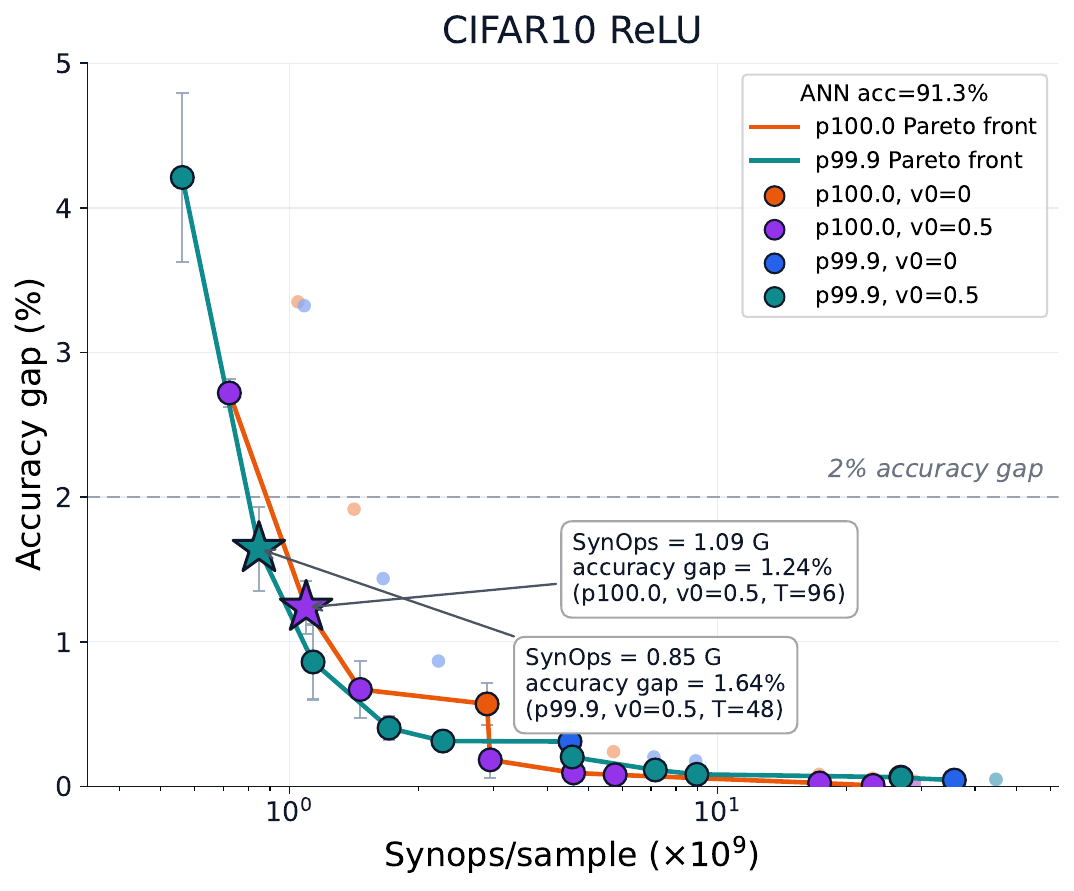}
  \caption{%
    Accuracy--cost trade-off for conventional IF conversion on
    VGG-11/CIFAR-10 with a ReLU source network, averaged over three seeds. The source ANN accuracy is $91.3\%$.
    Results are shown for $p100.0$ and $p99.9$ normalization with
    $v_0 \in \{0,0.5\}$ while sweeping the simulation length $T$. A separate
    Pareto front is shown for each normalization setting. The dashed line marks
    the $2\%$ ANN--SNN accuracy-gap criterion, and each star marks the
    lowest-SynOps configuration satisfying it. Under this criterion, the best
    $p99.9$ configuration reaches a $1.64\%$ gap at $0.85\,$G SynOps
    ($v_0{=}0.5$, $T{=}48$), while the best $p100.0$ configuration reaches a
    $1.24\%$ gap at $1.09\,$G SynOps ($v_0{=}0.5$, $T{=}96$).%
  }
  \label{fig:if-baseline-cifar10}
\end{figure}

\section{Additional Reporting Details}
\label{app:reporting}
\paragraph{Single-neuron finite-time convergence}
\begin{figure}[H]
  \centering
  \includegraphics[width=\linewidth]{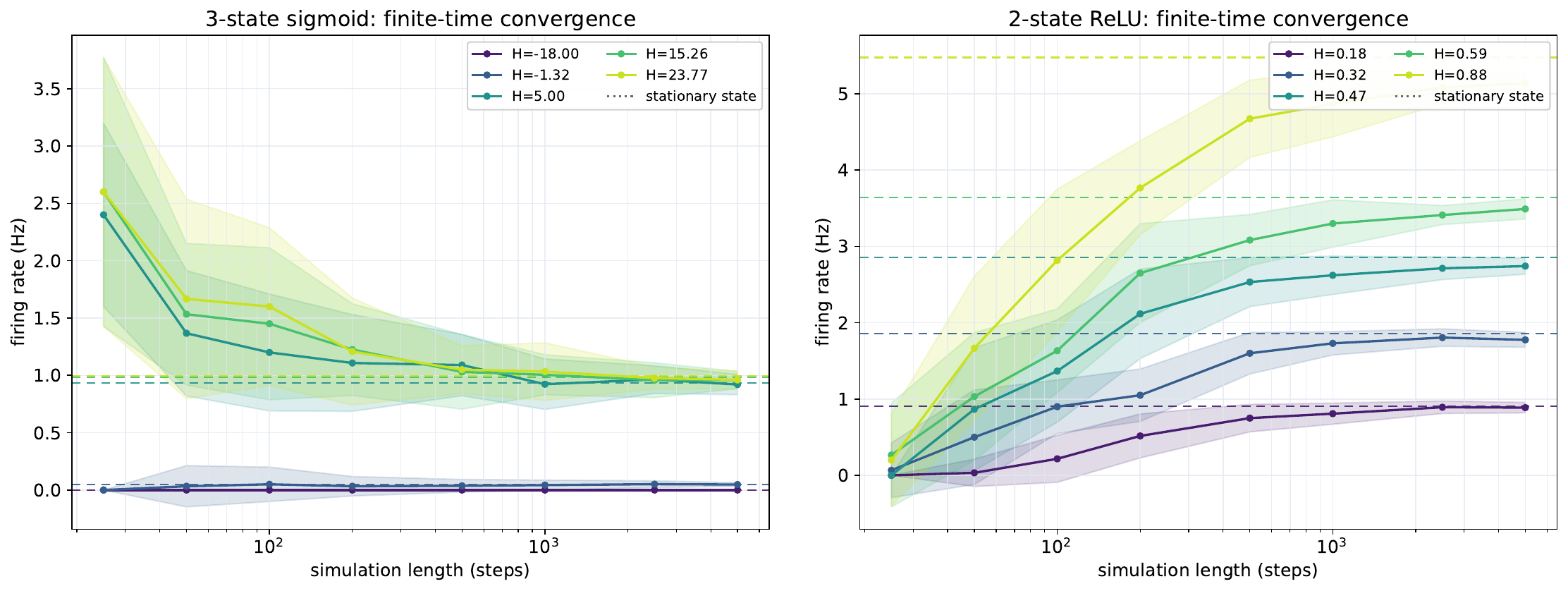}
  \caption{%
    Finite-time convergence of empirical firing rates to stationary rates for
    representative low-state CTMC neurons. Left: a 3-state sigmoid neuron evaluated at
    several input values $H$. Right: a 2-state ReLU neuron evaluated
    at several input values $H$. Shaded regions indicate variability across
    stochastic trials, and dashed horizontal lines mark the corresponding
    stationary firing rates. The results illustrate how finite simulation length
    contributes to the gap between stationary-rate predictions and actual SNN
    behavior in different time steps.
  }
  \label{fig:single-neuron-convergence}
\end{figure}
To illustrate the effect of finite simulation length, Figure~\ref{fig:single-neuron-convergence}
shows how empirical firing rates approach their stationary values for representative
2-state ReLU and 3-state sigmoid neurons.
\paragraph{Clipping interpretation.}
The MLP clipping ablation is ``clip-on-forward'': the same trained ReLU weights are evaluated after clamping every layer activation to $[0,K]$. This design controls the weights but changes both the ANN predictions and downstream preactivation distributions. Accordingly, the plotted accuracy gap answers whether the SNN tracks each altered forward map; it does not by itself show that the clipped system has higher absolute task accuracy at equal cost.

\paragraph{Deep operating points.}
The VGG-11/MNIST diagnosis labels Case I as $(T,\tau,r)=(50,0.02,3)$ and Case II as $(30,0.02,3)$. The CIFAR-10 diagnosis uses $(50,0.02,5)$ and $(40,0.02,4)$ respectively.

\paragraph{Transition-count cost.}
Figure~\ref{fig:cifar-transition-pareto} reports the corresponding
accuracy--cost trade-off on VGG-11/CIFAR-10 when CTMC state transitions,
rather than SynOps, are used as the cost measure. The qualitative trend is
consistent with the SynOps analysis: the clipped variant does not reduce
the minimum cost required to satisfy the $2\%$ accuracy-gap criterion.
\begin{figure}[t]
  \centering
  \includegraphics[width=\linewidth]{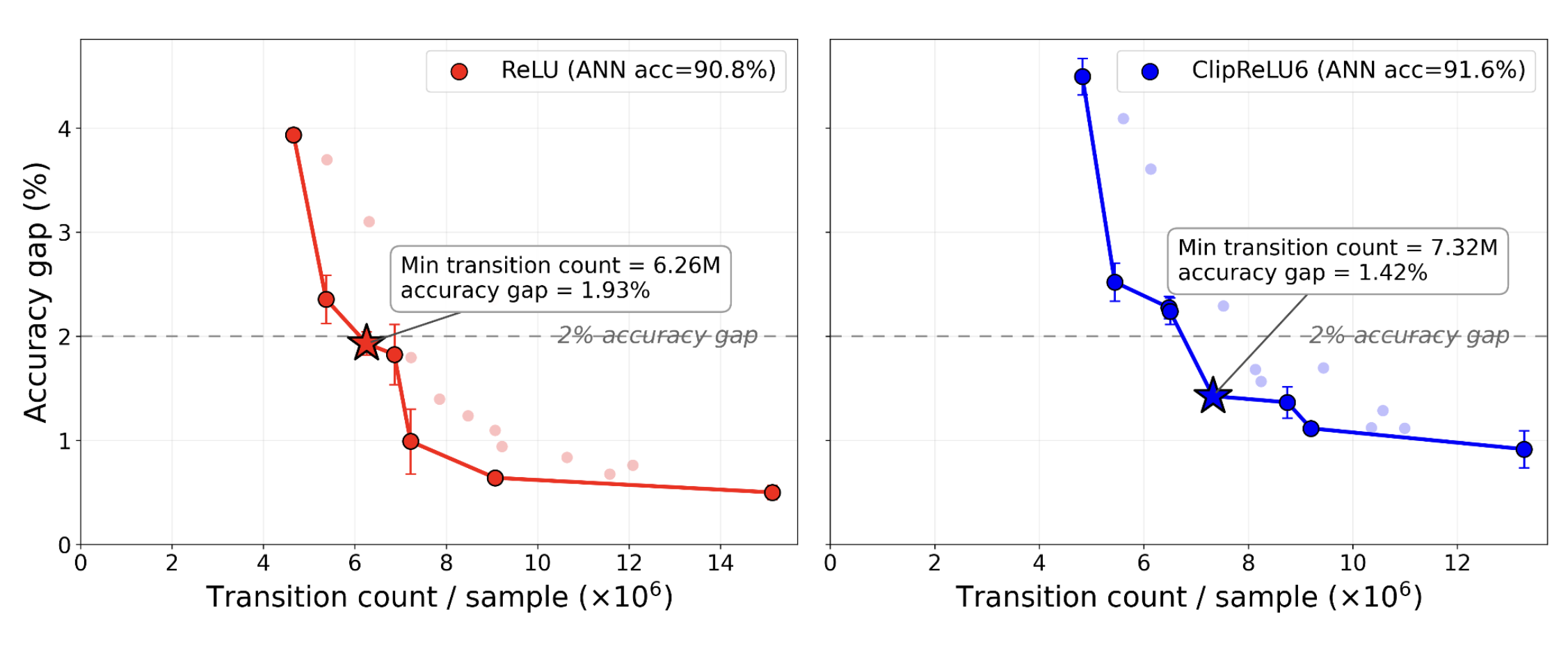}
  \caption{%
    Transition-count trade-off for CTMC conversion on VGG-11/CIFAR-10,
    comparing ReLU (left) and ClipReLU6 (right). Solid lines show the Pareto frontiers, and the minimum-transition configurations satisfying the $2\%$ accuracy-gap  criterion: ReLU
    reaches a $1.93\%$ gap at $6.26\times10^{6}$ transitions/sample, while
    ClipReLU6 reaches a $1.42\%$ gap at $7.32\times10^{6}$ transitions/sample.
  }
  \label{fig:cifar-transition-pareto}
\end{figure}
\paragraph{Uncertainty.}
For deep experiments, repeated stochastic trials share each trained ANN. Figures 5 and 6 report the SEM across the displayed stochastic runs and therefore primarily characterize run-to-run stochastic variability; trials sharing the same trained ANN are not independent training replicates.

\section{LLM Usage Disclosure}
A large language model was used to assist with editorial restructuring, prose refinement, LaTeX organization, reference-format checking, and reviewer-style critique. The authors are responsible for independently verifying every equation, citation, empirical value, interpretation, and disclosure.

\end{document}